\documentclass[12pt, final]{l4dc2022}

\DeclareMathOperator*{\rootfind}{RootFind}
\DeclareMathOperator*{\odesolve}{Integrator} 
\title[Symplectic Momentum Neural Networks]{Symplectic Momentum Neural Networks - Using Discrete Variational Mechanics as a prior in Deep Learning }

\usepackage{times}
\usepackage{hyperref}

\author{%
 \Name{Saul Santos} \Email{saul.r.santos@tecnico.ulisboa.pt}\\
 \Name{Monica Ekal} \Email{mekal@isr.tecnico.ulisboa.pt}\\
  \Name{Rodrigo Ventura} \Email{rodrigo.ventura@isr.ist.utl.pt}\\
 \addr Instituto Superior Técnico, Av. Rovisco Pais 1, 1049-001 Lisbon\\%
 \addr Instituto de Sistemas e Robótica, Instituto Superior Técnico, Av. Rovisco Pais 1, 1049-001 Lisbon
}

\begin{document}

\maketitle

\begin{abstract}%
With deep learning  gaining attention from the research community for prediction and control of real physical systems, learning important representations is becoming now more than ever mandatory. It is of extreme importance that deep learning representations are coherent with physics. When learning from discrete data this can be guaranteed by including some sort of prior into the learning, however, not all discretization priors preserve important structures from the physics. In this paper, we introduce Symplectic Momentum Neural Networks (SyMo) as models from a discrete formulation of mechanics for non-separable mechanical systems. The combination of such formulation leads SyMos to be constrained towards preserving important geometric structures such as momentum and a symplectic form and learn from limited data. Furthermore, it allows to learn dynamics only from the poses as training data. We extend SyMos to include variational integrators within
the learning framework by developing an implicit root-find layer which leads to End-to-End Symplectic Momentum Neural Networks (E2E-SyMo). Through experimental results, using the pendulum and cartpole, we show that such combination not only allows these models to learn from limited data but also provides the models with the capability of preserving the symplectic form and show better long-term behaviour. 
\end{abstract}


\section{Introduction}
Even though, Neural Networks have the ability to achieve astonishing results and being practically at the forefront of areas such as Computer Vision \cite{NIPS2012_c399862d} or natural language \cite{collobert_weston_2008}, these black-box parameterizations still lack of intuition about physics. In fact, traditional deep learning techniques rely on a search-and-evaluate procedure through a broad spectrum of candidate hypotheses and are limited to the characteristics of the data. By enforcing physical intuition to such strong models, explicitly or implicitly, is thus possible to narrow the hypothesis space to a range where these models have the ability to reproduce physically consistent results.

The application of even the state-of-the-art black box models has very frequently met with limited success in scientific domains due to their large data requirements, inability to produce physically consistent results, and their lack of generalizability to out-of-sample scenarios \cite{Willard2020}. By needing large amount of training data in order to generalize well to unseen states another problem comes up which is the fact that such data can be prohibitively expensive to obtain. 

In this work we are interested in combining physical priors with deep learning techniques for mechanical system modeling. Further, we extend this combination to account for numerical integration. Firstly, by combining continuous variational principles with discretization schemes by Neural Ordinary Differential Equations (NODE) \cite{chen2018neuralode} through conventional integrators. Secondly, the novelty resides on a family of geometric integrators called \textit{Variational Integrators} that have the particularity of preserving geometric structures of the continuous system \cite{marsden_west_2001}, in contrast to traditional integrators. Variational Integrators have the property of being symplectic and preserve the momentum (figure \ref{fig:phase_space_pendulum}), or its change in the presence of external forces \cite{johnson_schultz_murphey_2015}. Although symplectic-momentum integrators do not preserve energy exactly they exhibit bounded energy behavior achieving near energy conservation and thus they are extremely useful for long term simulations, such as for low thrust spacecraft missions \cite{junge}. Our contribution handles elegantly non-separable mechanical systems as we derive efficient back-propagation to grip with the implicit equations arising from the discrete variational principles. 

\begin{figure}[htb!]
  \centering
  \includegraphics[width=0.7\textwidth]{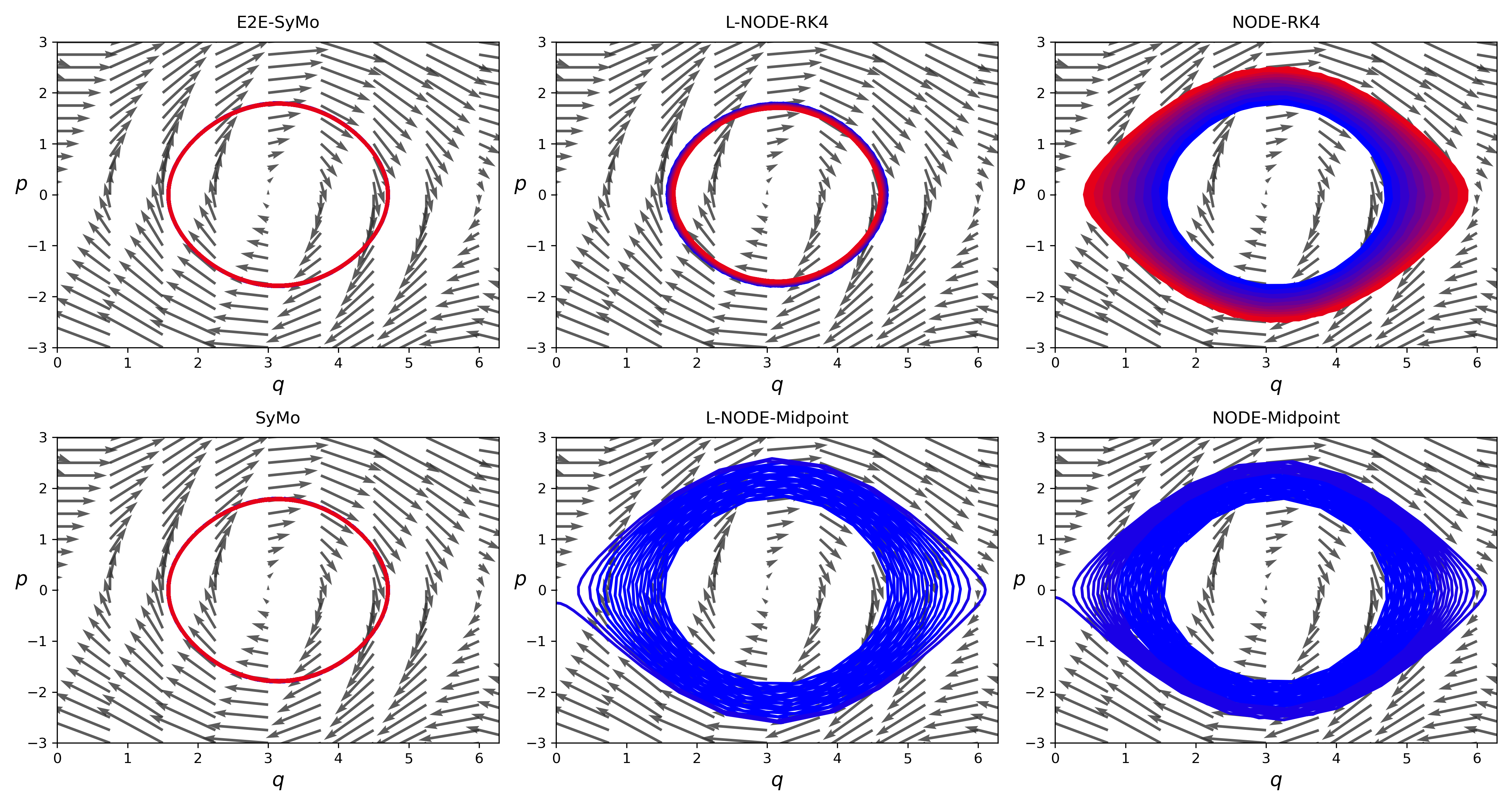}
  \caption[SyMo - Symplectic-Momentum Neural Networks.]{Phase Space of the simulated models for the pendulum system alongside the ground truth fields. Integration for 4000 time-steps with time-step $h=0.1$. NODE consists in a Neural ODE learning directly the dynamics from data with Runge kutta 4th-order and midpoint as ODE solvers. L-NODE consists of using the Lagrangian dynamics of mechanical systems as a model prior with the same ODE solvers of NODEs. SyMo results in the combination of discrete mechanics with deep learning. E2E-SyMo accounts for variational integration through a root finding implicit layer. SyMo and E2E-SyMo preserve the symplectic form while the other models drift away from the ground truth.}
  \label{fig:phase_space_pendulum}
\end{figure}

\section{Related Work}
Two recent developments were key to the developed approach. Firstly, in \cite{Raissi2017} use state-of-the-art deep learning techniques for solving Nonlinear Partial Differential Equations in a learning fashion with physical plausibility and good generalization properties. Secondly, \cite{Raissi2018} extended his work not only to learn the solution but also to discover the nonlinear partial differential equation. This lead researchers to try to mimic these approaches into a several range of applications. Specifically, \cite{Lutter2019}  proposed learning architectures that impose Lagrangian mechanics in a structured approach. This methodology is optimized to minimize the violation of the Euler-Lagrange Ordinary Differential Equation directly from data and further applied to energy-based control problems \cite{Lutter2019a}. This work is generalized to include all forms of Lagrangians in \cite{Cranmer2020} and extended to Hamiltonian mechanics in \cite{Greydanus2019}.

All the methods mentioned above assume continuous-time equations of motion for dynamical systems. These equations encode the underlying physical properties, such as symmetries corresponding to conservation laws. However, when learning from discrete data these continuous models need to be discretized. With this in mind, researchers started not only to combine physical priors with deep learning but also with numerical integrators. In \cite{chen2018neuralode} proposed NODEs, differentiable ode solvers with $O(1)$ memory backpropagation. This lead researchers to propose algorithms that combine methods with prior knowledge about the underlying equations of motion with geometric discretization schemes. Particularly in \cite{Zhong2019a}, authors proposed learning the dynamics of ordinary differential equation by combining Hamiltonian Mechanics with NODES and in \cite{srnn}, where authors leverage symplectic integration with neural networks. In \cite{Saemundsson2019} authors combine variational integrators with deep learning techniques. However, this approach assumes only separable Hamiltonians (inertia matrix does not depend on the configuration). Our work is capable of handling implicit equations arising from non-separable mechanical systems.

\section{Background}
\subsection{Continuous Lagrangian Dynamics}
Any mechanical system can be represented by a set of generalized coordinates, $q \in {\mathbb{R}^n}$, that define completely the configuration of the mechanical model. The Lagrangian mechanics define the scalar Lagrangian, $\mathcal{L}$,  as the difference between the kinetic energy, $T$, and  the potential energy, $V$. For rigid bodies the kinetic energy takes the  well-known quadratic form $T(q, \dot{q}) = \frac{1}{2}\dot{q}^TH(q)\dot{q}$ :
\begin{equation}
\label{lagrangian1}
    \mathcal{L} = T(q, \dot{q}) - V(q)= \frac{1}{2}\dot{q}^TH(q)\dot{q} - V(q),
\end{equation}
where $H \in \mathbb{R}^{n \times n}$ is called \textit{generalized inertia matrix} \cite{Murray94amathematical}, and by nature is represented as a positive definite matrix. 

In the forced case, one can extend the Hamiltonian principle to include non-conservative forcing such as external inputs or dissipative forces by using the Lagrange-d’Alembert principle \cite{Ober-Blobaum2011}, that seeks paths satisfying:
\begin{equation}
\label{action}
    \delta \int_{t_0}^{t_f} {\mathcal{L}(q, \dot{q})} \, dt 
    + \int_{t_0}^{t_f} {F(q, \dot{q}, u)} \,\delta q dt = 0,
\end{equation}
where the second integrand is defined as the virtual work and $\delta$ represents variations vanishing at the endpoints. Requiring that the variations of the action be zero for all $\delta q$, and integrating (\ref{action}) by parts, we arrive at the \textit{forced Euler-Lagrange equations}
\begin{equation}
    \label{F_EL}
  \frac{d}{dt} \left( \frac{ \partial \mathcal{L}(q, \dot{q}) }{\partial \dot{q}}\right)  - \frac{ \partial \mathcal{L}(q, \dot{q}) }{\partial q}  = F(q, \dot{q}, u),
\end{equation}
where $F$ is the forcing function that expresses the external forces in terms of generalized coordinates and $u$ are the input controls such as motor torques. 

The forced discrete Euler-Lagrange equations define an implicit equation that can be transformed into an explicit equation by taking into account the Lagrangian (\ref{lagrangian1}). Therefore, the obtained model can be reformulated into the \textit{forward manipulator equation},
\begin{equation}
\label{forward_manipulator}
   \ddot{q} = f^{(a)}(q, \dot{q}, u) = H^{-1}(q) \left( F(q, \dot{q}, u) - C(q,\dot{q}) + g(q)    \right),
\end{equation}
where $C \in \mathbb{R}^n$ is called the vector of \textit{coriolis} and \textit{centrifugal} forces, while $g(q) = \frac{\partial V(q)}{\partial q} \in \mathbb{R}^n$ is the \textit{gravitational potential vector}.
Equation (\ref{forward_manipulator}) can be used to obtain a trajectory $q(t)$ starting from an initial condition $x_{t_0} = [q(t_0), \dot{q}(t_0)]$ and actuation $u$, via numerical integration.

\subsection{Discrete Lagrangian Mechanics}
In an alternative approach of numerical integration, one can instead, discretize the action integral forcing the errors to respect the conservation of properties. This approach is called \textit{discrete mechanics} and the corresponding numerical integrators \textit{variational integrators} \cite{marsden_west_2001}.

Rather than having a configuration $q$ and velocity $\dot{q}$, in the context of discrete mechanics the trajectory of a continuous-time dynamical system is approximated by choosing a discrete Lagrangian, $\mathcal{L}_d$, that discretizes the action integral along the exact solution curve segment, $q$, i.e. $q_k \approx q(t_k)$ and $q_{k+1} \approx q(t_k+h)$  \cite{Ober-Blobaum2011}:   
\begin{equation}
    \mathcal{\mathcal{L}}_d(q_k, q_{k+1}, h) \approx \int_{t_k}^{t_k+h} {\mathcal{L}(q (t), \dot{q}(t))} \, dt .
\end{equation}
Similar to the continuous case, the discrete approach can also be extended to accommodate external forces by approximating the continuous forces by a left, $f_d^-$ and right discrete force, $f_d^+$\footnote{The left and right forces can be thought as the continuous control force acting during the time spans [$t_{k-1}$, $t_k$] and [$t_{k}$, $t_{k+1}$] \cite{Ober-Blobaum2011}.},
\begin{equation}
\label{virtual_work}
f_d^+ (q_k, q_{k+1}, u_k). \delta q_{k+1} + f_d^- (q_k, q_{k+1}, u_k).\delta q_{k} 
 \approx \int_{t_k}^{t_k+h} F (q(t), \dot{q}(t), u(t))\delta q \, dt
\end{equation}
where $u_k$ is the discretization of the continuous force inputs, ie.  $u_k \approx  u(t_k)$.
 
Analogous to the continuous case, and performing the calculus of variations, we arrive to the discrete equivalent of equation (\ref{F_EL}) that is known by \textit{Forced Discrete Euler-Lagrange Equations} (DEL):
\begin{equation}
\label{DEL}
        D_1\mathcal{L}_d(q_k, q_{k+1})  +  D_2\mathcal{L}_d(q_{k-1}, q_{k}) + \\ 
         f_d^+ (q_{k-1}, q_k, u_{k-1}) + f_d^- (q_k, q_{k+1}, u_{k}) = 0
\end{equation}
where the slot derivative $D_i$ represents the derivative of the
discrete Lagrangian, $\mathcal{L}$, with respect to the $i^{th}$ argument. Variational Integrators are a class of algorithms with the intuit of solving the implicit DEL equation for $q_{k+1}$.

\section{Baselines}
\textbf{Geometric Neural Ordinary Differential Equation:} The first challenge towards leveraging Neural ODEs to learn state-space models of mechanical systems is the fact that the dynamics arising from such dynamical systems are defined by a second order ordinary differential equation. However, the dynamics can, instead, be seen as a system of coupled first-order ODEs with state $x = [q, \dot{q}]^T$ and parameters $\theta$:
\begin{equation}
    \dot{x} = f^{(\dot{x})}(x, u; \theta) = \begin{bmatrix}
\dot{q}\\
f^{(a)}(x, u, \theta)
\end{bmatrix}.
\end{equation}
With actuation set to be constant between time-steps the final augmented state assumes the form of
\begin{equation}
\label{eq:augmentation}
    \begin{bmatrix}
\dot{x}\\
\dot{u}
\end{bmatrix} =     \begin{bmatrix}
f^{(\dot{x})}(x, u; \theta)\\
0
\end{bmatrix} = f(x, u; \theta).
\end{equation}
With the regression prediction being given as an integration between two time steps 
\begin{equation}
\hat{x}_{t+h} = \odesolve \bigl(f(x, u; \theta)\bigr).    
\end{equation}
We consider observation of the form of $(x_t, u, x_{t+1})$. However, we rely, before feeding the neural networks with the observations, on an embedding where we convert angles $\theta$ to the circle manifold $(\cos \theta, sin \theta)$, hence the term geometric.
The augmented state can then be used by a NODE \cite{chen2018neuralode} and trained by optimizing the parameters of $f_{\theta}$ through minimization of the mean squared error loss  $L=\sum_{i=1}^{N} \lVert \hat{x}_{i, t+h} - x_{i, t+h}\rVert_2 ^2$ over integration between two adjacent time-steps.

\textbf{Lagrangian Neural Ordinary Differential Equation:} Here, we parameterize the potential energy $\hat{V}(q)_{\psi}$ and the inertia matrix $\hat{H}(q)_{\beta}$ and build the Euler-Lagrange equations (\ref{forward_manipulator}) with the embedding working as input of the neural network, without changing the nature of the underlying equations, since by the chain rule gradients are calculated on the generalized coordinates. We designate these models by L-NODE hereinafter.

\section{Symplectic Momentum Neural Networks}
\subsection{Problem Formulation}
Given input and output spaces $\mathcal{X} \times \mathcal{Y}$, we build a neural network based architecture over the $(x, y)$ pairs that encodes a function, $g(x, y; \theta)$, that relates those pairs with the discrete Euler-Lagrange equations. Let the observations be of the form of $(q_{k-1}, q_k, q_{k+1}, u_{k-1}, u_k, u_{k+1})$, where $q_{k+1}$ are the regression targets. The input space can be defined by a set of two adjacent points in the configuration space and three adjacent discrete control inputs $x=(q_{k-1}, q_k, u_{k-1}, u_k, u_{k+1}) \in \mathcal{X}$ and the output space to be $y=(q_{k+1}) \in \mathcal{Y}$. Our goal is to build a neural network architecture that captures dependencies between the input and output space by using the Discrete Euler-Lagrange equation as a function that correlates both spaces.

We intent to build equation (\ref{DEL}) based on a neural network representation. Hence, learning consists in minimizing the violation of the DEL equations. This can be obtained by adjusting the free parameters, $\theta$ to minimize the loss function  $L_{\theta}=\frac{1}{N} \sum_{i=1}^{N} \lVert g(x_i,y_i; \theta)\rVert_2 ^2$.

Inference consists in finding configurations of the variables $q_{k+1}$, obtained by implicitly solving the parameterized DEL, through the Newton's root finding algorithm:
\begin{equation}
    \hat{q}_{k+1} = \rootfind \bigl(g(q_{k+1}; \theta^*)\bigr).
\end{equation}
We designate this models by Symplectic Momentum Neural Networks (SyMo), henceforth.
\subsection{Ensuring Symmetry and Positive Definiteness of the Inertia Matrix}
In mechanical systems the inertia matrix, $H$ is symmetric positive definite. Symmetry and positive semi-definetess is enforced by assuming a parameterized lower triangular matrix $\hat{L}(q; \beta)$ by a neural network, which predicts its $\frac{N^2 + N}{2}$ elements of its Cholesky factor with $N$ being the number of degrees of freedom of the mechanical system, then the Cholesky decomposition of this matrix is a decomposition $\hat{H}(q; \beta)= \hat{L}(q; \beta) \hat{L}^T(q; \beta)$ \cite{Lutter2019}. Positive definiteness is obtained if the diagonal of $\hat{L}$ is positive. This is guaranteed by using a non-negative activation function such as the SoftPlus for the diagonal elements of $\hat{L}$ and adding a small constant $\epsilon$ to the diagonal elements of $\hat{H}$. 

\subsection{Discretization}
In this work we consider the midpoint rule as an approximation of the action integral
\begin{equation}
\label{eq:discrete_lagrangian}
\mathcal{L}_d(q_k, q_{k+1}; \theta) = h \mathcal{L}\left(\frac{q_k + q_{k+1}}{2}, \frac{q_{k+1} - q_k}{h}; \theta \right) \approx \int_{0}^{h} {\mathcal{L}(q(t), \dot{q}(t))} \, dt.
\end{equation}
The above approximation leads to second-order accuracy \cite{marsden_west_2001}. From now on, we will denote the collocation points by $q_{k+1/2} = \frac{1}{2}q_k + \frac{1}{2}q_{k+1}$ and $q_{k-1/2} = \frac{1}{2}q_{k-1} + \frac{1}{2}q_k$.
For Lagrangians of Mechanical Systems (see equation (\ref{lagrangian1})) and using the neural network architecture aforementioned to parameterize the inertia matrix at the collocation points $q_{k-1/2}$ and $q_{k+1/2}$, through the lower triangular matrix $L(q_{k \pm 1/2}; \beta)$, and the potential energy, $V(q_{k \pm 1/2}; \psi)$, the discrete Lagrangians for the two adjacent time-steps are built upon the special structure of the approximated Lagrangian integral,
\begin{subequations}
\begin{equation}
\mathcal{\hat{L}}_d(q_{k-1}, q_{k}; \theta) = h \left [\left( \frac{q_{k} - q_{k-1}}{h}\right)\hat{H}\big(q_{k-1/2}; \beta \big) \left( \frac{q_{k} - q_{k-1}}{h}\right)^T - \hat{V}\big((q_{k-1/2}; \psi \big) \right],
\end{equation}
\begin{equation}
\mathcal{\hat{L}}_d(q_k, q_{k+1}; \theta) = h \left [\left( \frac{q_{k+1} - q_k}{h}\right)\hat{H}\big(q_{k+1/2}; \beta \big) \left( \frac{q_{k+1} - q_k}{h}\right)^T - \hat{V}\big((q_{k+1/2}; \psi \big) \right].
\end{equation}
\end{subequations}
Similarly the collocation points for the control inputs can be denoted by $u_{k+1/2} = \frac{1}{2}u_k + \frac{1}{2}u_{k+1}$
and $u_{k-1/2} = \frac{1}{2}u_{k-1} + \frac{1}{2}u_k$.
We can approximate the virtual work (equation \ref{virtual_work}) by the midpoint rule for control affine systems:
\begin{subequations}
\begin{equation}
    f_d^-(q_k, q_{k+1}, u_k, u_{k+1}) = \frac{h}{2}F\left(q_{k+1/2},\frac{q_{k+1}- q_k}{h}, u_{k+1/2}\right)= \frac{h}{4}(u_{k+1}+u_k),
\end{equation}
\begin{equation}
    f_d^+(q_{k-1}, q_k, u_{k-1}, u_k) = \frac{h}{2}F\left(q_{k-1/2},\frac{q_{k}- q_{k-1}}{h}, u_{k-1/2}\right) = \frac{h}{4}(u_{k-1}+u_k).
\end{equation}
\end{subequations}
Formally, we are now able to write the Discrete-Euler Lagrange equations in terms of the learning parameters, input and output spaces by:
\begin{equation}
\label{eq:DEL_NN}
    g_{\theta} = \frac{\partial}{ \partial q_k} \big[ \mathcal{\hat{L}}_d(q_{k-1}, q_k; \theta) + \mathcal{\hat{L}}_d(q_k, q_{k+1}; \theta) \big] + f_d^-(q_k, q_{k+1}, u_k, u_{k+1}) + f_d^+(q_{k-1}, q_k, u_{k-1}, u_k),
\end{equation}
All the gradient operations are performed by automatic differentiation.

\textbf{Learning from Embeddings:} Considering observations in $R^n$ we apply appropriated embeddings to the observations that follow angular coordinates, making the learning free of singularities. Here, we use an  embedding $s(q)$ but on the collocation points, $q_{k\pm1/2} \xrightarrow[]{} s(q_{k\pm1/2})$ to account for rotational coordinates. The embedding is not applied to translational coordinates.

\section{End-to-End Symplectic Momentum Neural Networks}
So far, we have established a neural architecture that takes into account the underlying discrete equations of motion arising from the discrete variational principles. However, we have not discussed yet the process of including the resulting integrators into the learning process. Assuming a loss $L_{\theta}=\frac{1}{N} \sum_{i=1}^{N} \lVert \hat{q}_{k+1}-q_{k+1}\rVert_2 ^2$, we need the gradient of the implicit solution $\hat{q}_{k+1}$ with respect to $\theta$. This is trivial to obtain by using the implicit function theorem and gives origin to End-to-End Symplectic Momentum Neural Networks (E2E-SyMo) (figure \ref{fig:symo}). In fact, theorem \ref{backward} gives us an insight of how to implicitly compute the gradients necessary for back-propagation without the need to back-propagate through all intermediate operations of the root-finding algorithm.
\begin{theorem}
\label{backward}
\textbf{Gradient of the RootFind solution.} Let $q_{k+1}\in \mathbb{R}^n$ be the root of the physical constrained parameterized RootFind procedure based on the implicit DEL equations, $g(q_{k+1}, x; \theta) \in \mathbb{R}^n$ (equation \ref{eq:DEL_NN}), mapping $(q_{k-1}, q_k)$ to $(q_k, q_{k+1})$. The gradients of the scalar loss function $L_{\theta}$ with respect to the parameters $\theta$ are obtained by vector-Matrix products as follows:
\begin{equation}
\label{eq:back_grad}
    \frac{\partial L}{\partial \theta} = -\frac{\partial L}{\partial q_{k+1}} {\underbrace{\left [\frac{\partial \mathcal{L}_d(q_k,  q_{k+1}; \theta)}{\partial q_{k+1} \partial q_k}  + \frac{\partial f_d^-(q_k, q_{k+1}, u_k, u_{k+1})}{\partial q_{k+1}}\right ]^{-1} \frac{\partial g(q_{k+1}; \theta)}{\partial \theta}}_{-\frac{\partial q_{k+1}}{\partial \theta}}}
\end{equation}
\end{theorem}
\begin{proof}
Under the conditions of the implicit function theorem, we can formulate the relation between the root $q_{k+1}^*$ and the remaining variables as the graph of a function of the remaining variables, $q_{k+1}^* = f(x; \theta)$, that define the implicit DEL equations. 

The theorem also tell us how to compute derivatives of $q_{k+1}(x; \theta)$ by implicit differentiation, and based on that obtain the partial derivative of $q_{k+1}$ w.r.t. $\theta$.
  \begin{align*}
    g\left(x, q_{k+1} ; \theta \right) &= 0 \\
    \frac{\partial}{\partial \theta} \bigl[ g\left(x, q_{k+1}; \theta\right)\bigr] &= 0 && \text{Differentiate both sides.}\\
    \frac{\partial g(x, q_{k+1})}{\partial \theta} + \frac{\partial g(x, q_{k+1}; \theta)}{\partial q_{k+1}} \cdot \frac{\partial q_{k+1}}{\partial \theta} &= 0 &&\text{Using the Chain Rule.} \\
     -\left[ \frac{\partial g(x, q_{k+1}; \theta)}{\partial q_{k+1}}\right]^{-1} \frac{\partial g(x, q_{k+1})}{\partial \theta} &= \frac{\partial q_{k+1}}{\partial \theta}   &&\text{Rearranging the equations.}
  \end{align*}
\end{proof}

\textbf{Forward Pass.} Contrary to SyMos where the output is the DEL equations, here, the output are the roots of the parameterized DEL equations solved by Newton's root-finding algorithm, i.e., the output is
\begin{equation}
    q_{k+1} = \rootfind_{q_{k+1}} \bigl(g(x, q_{k+1}; \theta)\bigr).
\end{equation}

This infinite depth RootFind layer adds two new hyperparameters. Specifically, the tolerance and the maximum number of iterations of the Newton's method. The initial guess for the algorithm is employed by explicit Euler integration, i.e, $q^{(0)}_{k+1} = q_k + h \dot{q}_k$, with $\dot{q}_k = \frac{1}{h} (q_{k} -q_{k-1})$, which is simply $q^{(0)}_{k+1} = 2q_k - q_{k-1}$.

\textbf{Backward Pass.} The backward pass simply consists in applying the insight given by theorem \ref{backward}, with gradients being calculated through automatic differentiation.

\begin{figure}[t]
  \centering
  \includegraphics[width=1\textwidth]{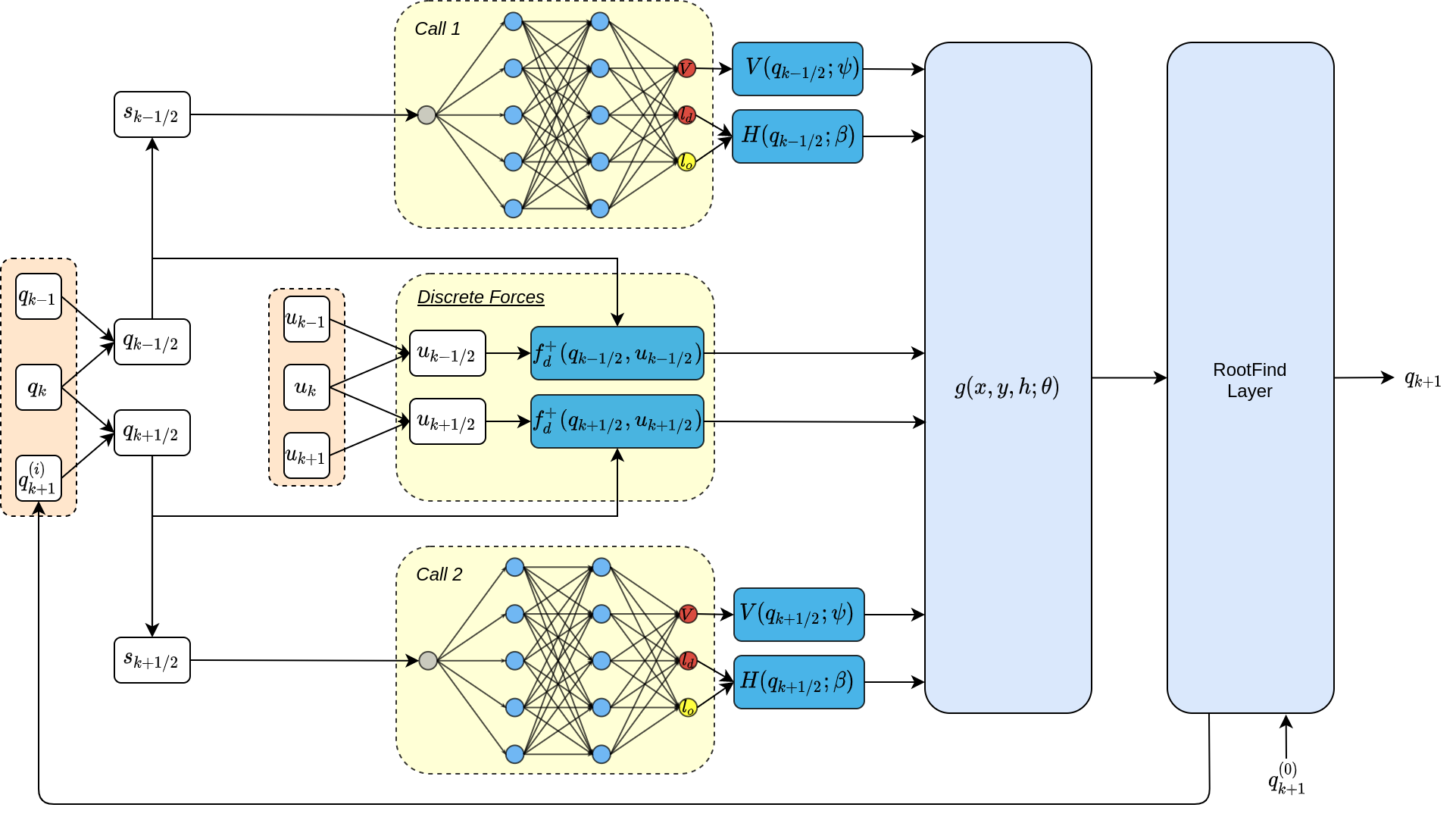}
  \caption[SyMo - Symplectic-Momentum Neural Networks.]{E2E Symplectic Momentum Neural Network - Given the embeddings $s$, we perform two calls to the neural network to get the inertia and potential energy correspondent to the two adjacent time-steps. These terms, alongside the discrete forces form then the discrete Euler-Lagrange equations, $g(x, y, h; \theta)$ as a function of the learning parameters. The DEL are then used by the implicit layer defined by the root finding procedure. Given an initial estimate $q_{k+1}^{(0)}$ the rootfind layer iterates over the DEL equations to obtain $\hat{q}_{k+1}$.}
  \label{fig:symo}
\end{figure}

\section{Experimental Setup}
We use two simulated dynamical systems to evaluate the performance of SyMos when compared to the baselines arising from conventional integrators. Specifically, we use the physical pendulum and the cartpole as systems.

\textbf{Dataset Generation:} The initial state is composed by the configurations and velocities $x_0=[\theta, \Dot{\theta}]$. We randomly generated initial conditions with a uniform distribution.  The initial conditions are combined with a constant control input for each trajectory in a uniform distribution in the range of $u \in [-2,2]$.  The ground truth trajectories are simulated by SciPy’s \textit{solve\_ivp} adaptive solver \cite{scipy} with method RK45 using Open AI Gym \cite{openai}. In order to show that the methods that incorporate priors can learn from limited amount of data we vary the size of the training set by doubling from 8 to 128 initial state conditions. Each trajectory is integrated for 32 time steps with a time span of $h=0.1$ for the pendulum and $h=0.05$ for the cartpole. 

\textbf{Methods:} For the NODE and L-NODE we choose the “RK4” and the midpoint methods of numerical integration for training and making predictions. We use the labels NODE-RK4, NODE-Midpoint, L-NODE-RK4 and L-NODE-Midpoint to describe the integrator used. We set the size of mini-batches to be four times the number of initial state conditions. Besides the Softplus for the diagonal elements of $\hat{L}$ we use linear activations in the output layer. At the hidden level we use the hyperbolic tangent. The neural network consists of a two-hidden-layer with 128 hidden neurons for the pendulum and 256 for the cartpole.

\textbf{Metrics:} To have a fair comparison with SyMos and once learning the L-NODEs and NODEs involves velocities, we set the train and test error as the Mean Squared Error (MSE) between the estimated configurations and the ground truth. The test loss is based on 128 test trajectories where each is integrated for 32 time-steps. To evaluate the performance of each model in terms of long term prediction, we construct the metric of integration error by using 16 random unseen initial state conditions without control actuation and integrating them for 500 time-steps. We logged the inertial loss for the models that make usage of physical priors through the inertia matrix, where the loss is defined by the MSE between the predicted inertia matrix and the ground truth. We also logged the energy loss per trajectory for the 16 trajectories.

\textbf{Model Training:} We train our models using Adam optimizer \cite{Kingma2015} for 2000 epochs with initial learning rate of 0.0001 for the pendulum and of 0.001 for the cartpole and using the \textit{ReduceLROnPlateau} scheduler\footnote{\url{https://pytorch.org/docs/stable/generated/torch.optim.lr_scheduler.ReduceLROnPlateau.html}} with patience 50 and factor 0.7. For the E2E-SyMo we use a training tolerance of $1e-5$ and a maximum number of 10 iterations. For integration we relax the tolerance to $1e-4$. For SyMo, after training, we perform the root finding operation to the learned DEL equations with the same hyperparameters.
\section{Results}
\begin{figure}[htb!]
  \centering
  \includegraphics[width=1\textwidth]{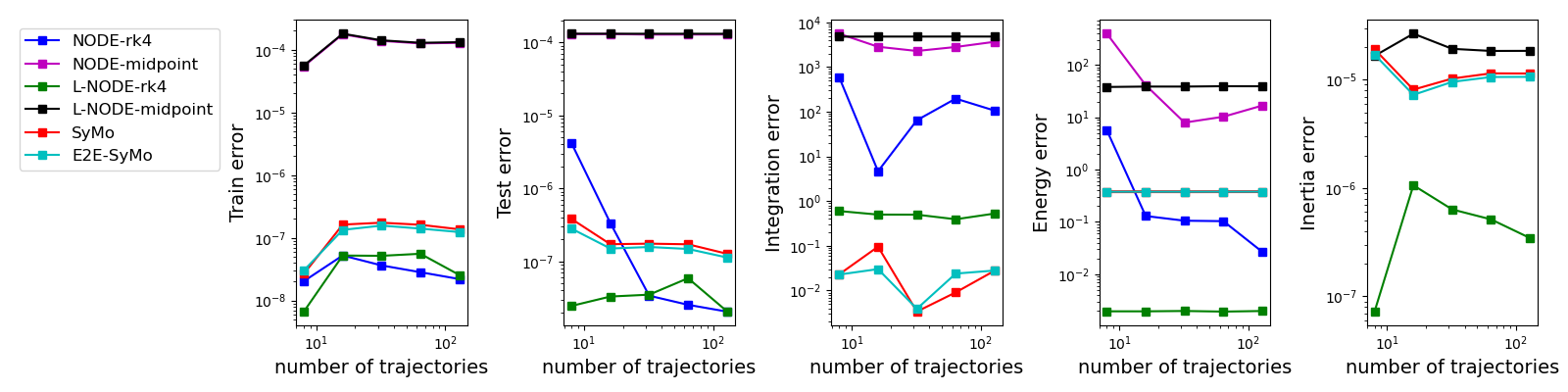}
  \caption[SyMo - Symplectic-Momentum Neural Networks.]{Mean Squared error in the configuration space for training, testing and integration and energy and mass MSE for the integration trajectories for the pendulum. SyMos show a lower integration error.}
  \label{fig:train_test_pendulum}
\end{figure} 

Figure \ref{fig:train_test_pendulum} shows the variation in train error, test error, integrator error, energy error and inertia error with changes in the number of initial state conditions in the training set. We can see that the incorporation of prior knowledge generally yields better train and test losses, specially for a small number of trajectories. For instance, the NODE-RK4 requires more trajectories to achieve similar results with the SyMos and L-NODE-RK4. This means that the priors provide the models with the capability of learning the dynamics with less training data. However, in terms of integration, NODE-RK4 fails to keep up with those same models. This is normal as the absence of priors do not provide the models with the capability of preserving important physical properties. SyMos have a lower integration error which emphasizes their good long-term behaviour. This is expected as the priors take into account the geometric structure of the original continuous time equations leading to better integration behaviour.

Figure \ref{fig:train_pred_cartpole} shows the results for the cartpole system. E2E-SyMo present the lowest test, integration and inertia error. Even outperforming the models with fourth-order integrators which shows how important is the preservation of geometric structures when integrating these models. In terms of energy conservation L-NODE-RK4 outperform the remainder methods. NODE-RK4 beats all the other models in terms of train loss but it's outperformed in the other metrics. This shows that the models without a prior are prone to overfitting.
\begin{figure}[htb!]
  \centering
  \includegraphics[width=1\textwidth]{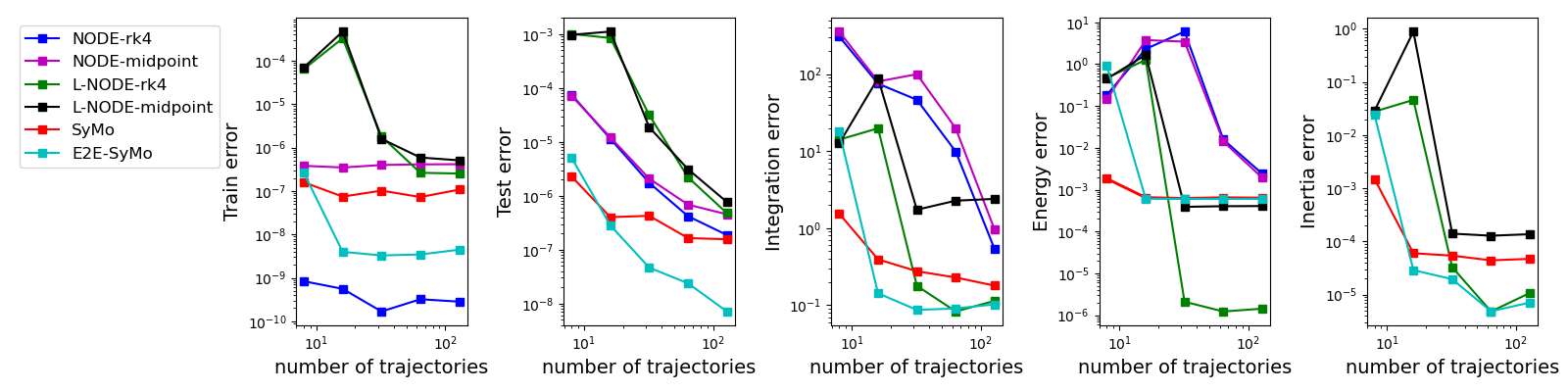}
  \caption[SyMo - Symplectic-Momentum Neural Networks.]{Mean Squared error in the configuration space for training, testing and integration and energy and inertial MSE for the integration trajectories for the cartpole. SyMos show a better integration behaviour and inertia modeling error.}
  \label{fig:train_pred_cartpole}
\end{figure}

\section{Conclusion}
In this work we introduced Symplectic Momentum Neural Networks as architectures that incorporate discrete mechanics with deep learning techniques. Further, we derived an implicit layer to accommodate variational integrators within the learning framework for systems where the inertia depends on the configuration. We compared these models against the Neural ODEs with and without prior knowledge about the underlying equations of motion. Our models showed better long-term behaviour and conservation of properties, proving that these models can be used for data driven numerical integration in an efficient and interpretable way. Future work should address the incorporation of dissipation into the learning framework. Code is made available in \url{https://github.com/ssantos97/SyMo}.

\acks{This work was supported by the LARSyS - FCT Project UIDB/50009/2020 and FCT PhD grant ref. PD/BD/150632/2020}

\bibliography{bib}

\end{document}